\documentclass[journal]{IEEEtran}
\usepackage{amsmath}
\usepackage{amsthm}
\usepackage{kotex}
\usepackage{amsfonts}
\newtheorem{thm}{Theorem}[section]

\newtheorem{prop}[thm]{Proposition}
\newtheorem{asum}[thm]{Assumption}

\usepackage{romannum}
\usepackage{graphicx}
\usepackage{caption}
\usepackage{subcaption}
\usepackage{breqn}
\usepackage{algorithm}
\usepackage{algpseudocode}
\usepackage{rotating}
\usepackage{pdflscape}
\usepackage{caption}
\usepackage{booktabs}
\usepackage{xcolor}

\begin{document}
\title{Hybrid Federated Learning \\for Noise-Robust Training}

\author{Yongjun~Kim,~\IEEEmembership{Graduate Student Member,~IEEE,}
        Hyeongjun~Park,~\IEEEmembership{Graduate Student Member, ~IEEE}
        Hwanjin~Kim,~\IEEEmembership{Member, IEEE}
        and~Junil~Choi,~\IEEEmembership{Senior Member, IEEE}
\thanks{
Y. Kim, H. Park, and J. Choi are with the School of Electrical Engineering, KAIST, Daejeon 34141, Korea (email: {yongjunkim, mika0303, junil}@kaist.ac.kr).}
\thanks{H. Kim is with the School of Electronics Engineering, Kyungpook National University, Daegu 41566, Korea (email: hwanjin@knu.ac.kr)}
}


\maketitle

\begin{abstract}
Federated learning (FL) and federated distillation (FD) are distributed learning paradigms that train UE models with enhanced privacy, each offering different trade-offs between noise robustness and learning speed. To mitigate their respective weaknesses, we propose a hybrid federated learning (HFL) framework in which each user equipment (UE) transmits either gradients or logits, and the base station (BS) selects the per-round weights of FL and FD updates. We derive convergence of HFL framework and introduce two methods to exploit degrees of freedom (DoF) in HFL, which are (i) adaptive UE clustering via Jenks optimization and (ii) adaptive weight selection via a damped Newton method. Numerical results show that HFL achieves superior test accuracy at low SNR when both DoF are exploited.
\end{abstract}


%
\IEEEpeerreviewmaketitle

\section{Introduction}
\label{Sec:Introduction}

Next-generation wireless networks are expected to process data from a massive number of devices with low latency. However, centralized learning is often impractical due to strict privacy regulations and limited backhaul capacity \cite{Li:2020}. Federated learning (FL) \cite{McMahan:2017} addresses these challenges by keeping training data local and exchanging only model parameters or gradients. This decentralized paradigm offers three benefits: (i) enhanced privacy by avoiding raw data exposure, (ii) reduced communication overhead through lightweight transmissions, and (iii) lower latency for time-critical 5G/6G applications such as autonomous driving \cite{Lim:2020}. 

Despite these advantages, FL remains challenging over wireless links. High-dimensional gradients create substantial uplink and downlink traffic and impose computational and energy burdens on low-power edge devices. To alleviate this, federated distillation (FD) \cite{Jeong:2018, Li:2024} exchanges logits instead of parameters, yielding smaller messages, lower communication and energy costs, stronger privacy in stricter regulatory settings, and support for distributed learning without requiring labels at each client. 

While FD has emerged as a promising approach for intelligent edge services, it sacrifices peak accuracy in exchange for lower communication overhead and improved robustness to uplink noise \cite{Mu:2024, Gad:2024}.
Under noiseless uplink conditions, exchanging only low-dimensional logits generally yields lower accuracy than FL \cite{Lin:2020}, because when uplink noise corrupts transmissions, FL suffers significant degradation because high-dimensional gradients are directly distorted. FD, by transmitting lower-dimensional logits that are less sensitive to noise, is more robust for power-constrained edge devices \cite{Ahn:2019}.

To address the respective limitations of FL and FD, we propose a hybrid federated learning (HFL) framework in which each user equipment (UE) adaptively transmits either gradient or logit in each communication round, retaining FL's accuracy under favorable channel conditions while leveraging FD's noise robustness under adverse conditions. The main contributions of this work are summarized as follows:
\begin{enumerate}
    \item We design a novel HFL framework that jointly leverages gradients and logits transmissions, and we provide a convergence analysis.
    \item We introduce methods that optimize transmission degrees of freedom (DoF) via adaptive UE clustering and weight selection.
\end{enumerate}
The remainder of the paper is organized as follows. Sec.~\ref{Sec:System model} presents the system model, including signal transmission and reception. Sec.~\ref{sec:hfl} details the HFL scheme, its convergence analysis, and methods for adaptive UE clustering and weight selection. Sec.~\ref{Sec:Simulation results} provides simulation results validating the noise-robust test accuracy of HFL at low signal-to-noise ratio (SNR). Finally, Sec.~\ref{Sec:Conclusion} concludes the paper.

\section{System Model}
\label{Sec:System model}

We consider a cooperative learning system in which a base station (BS) with $N$ antennas serves a set of $K$ single-antenna UEs $\mathcal{K}$, i.e., $|\mathcal{K}| = K$. Each UE $k \in \mathcal{K}$ conveys a local training dataset $\mathcal{D}_{k}$. At communication round $t$, the BS sends an indicator $\mathcal{I}_k^{(t)} \in \{0,1\}$ specifying the uplink transmission data type for UE $k$. Define the FL set $\mathcal{K}_1 = \{k_1: k_1 \in \mathcal{K},  \mathcal{I}_{k_1}^{(t)} = 0\}$ with $|\mathcal{K}_1| = K_1$ and the FD set $\mathcal{K}_2 = \{k_2: k_2 \in \mathcal{K},  \mathcal{I}_{k_2}^{(t)} = 1\}$ with $|\mathcal{K}_2| = K - K_1$. UE $k_1$ in $\mathcal{K}_1$ transmits gradients $\mathbf{g}_{k_1}^{(t)} \in \mathbb{R}^{P}$, where $P$ is the length of the vectorized gradient. UE $k_2$ in $\mathcal{K}_2$ computes logits on a public dataset $\mathcal{D}_{\text{pub}}$ containing $P_{\text{pub}}$ inputs, yielding $\mathbf{z}_{k_2}^{(t)} \in \mathbb{R}^{C P_{\text{pub}}}$ for $C$ classes in the classification task.

For wireless transmission, both uplink transmission data types are mapped to the complex signal through the following processing. Let $\mathbf{u}_k^{(t)}$ denote either $\mathbf{g}_{k_1}^{(t)}$ or $\mathbf{z}_{k_2}^{(t)}$. First, form a complex-valued vector by pairing real entries as $\tilde{\mathbf{u}}_k^{(t)}[m] = \mathbf{u}_k^{(t)}[2m - 1] + j\, \mathbf{u}_k^{(t)}[2m]$ for $1 \le m \le \lceil |\mathbf{u}_k^{(t)}| / 2 \rceil$, where $\lceil \cdot \rceil$ is the ceiling function and $[m]$ denotes the $m$-th element. To ensure consistent scaling, standardization is applied as $\tilde{\tilde{\mathbf{u}}}_k[m] = (\tilde{\mathbf{u}}_k[m] - \mu_{\tilde{\mathbf{u}}_k})/\sigma_{\tilde{\mathbf{u}}_k}$, where $\mu_{\tilde{\mathbf{u}}_k}$ and $\sigma_{\tilde{\mathbf{u}}_k}$ are the mean and standard deviation of $\tilde{\mathbf{u}}_k$, respectively, followed by normalization $\tilde{\tilde{\tilde{\mathbf{u}}}}_k[m] = \tilde{\tilde{\mathbf{u}}}_k[m]/\lVert \tilde{\tilde{\mathbf{u}}}_k \rVert_\infty$. We set $L =\max_k |\tilde{\tilde{\tilde{\mathbf{u}}}}_k|$ as the number of time slots per round. If the length of $\tilde{\tilde{\tilde{\mathbf{u}}}}_k$ is shorter than $L$, zero-padding is applied to unify the signal length, resulting in the final transmit signal of UE $k$ $\mathbf{x}_{k}^{(t)} \in \mathbb{C}^{1 \times L}$. Stacking all UEs’ signals yields $\mathbf{X}^{(t)} \in \mathbb{C}^{K \times L}$, which is transmitted over the wireless channel. To focus on the effect of noise in $\mathbf{X}^{(t)}$, we assume error-free transmission of $\mu_{\tilde{\mathbf{u}}_k}$, $\sigma_{\tilde{\mathbf{u}}_k}$, $\lVert \tilde{\tilde{\mathbf{u}}}_k \rVert_{\infty}$, and all downlink signals.

When transmitting $\mathbf{X}^{(t)}$, a Rayleigh fading channel $\mathbf{H}^{(t)} = [ \mathbf{h}_1^{(t)}, \mathbf{h}_2^{(t)}, \ldots, \mathbf{h}_K^{(t)} ] \in \mathbb{C}^{N \times K}$ is considered. 
We assume a constant channel within each communication round. Each communication round $t$ consists of $L$ synchronized time slots indexed by $m = 1,2,\cdots,L$, during which each column of $\mathbf{X}^{(t)}$ is transmitted.
The received signal at the BS in the $m$-th time slot is
\begin{align}
\mathbf{y}^{(t)}[m] = \sqrt{\rho}\, \mathbf{H}^{(t)} \mathbf{x}^{(t)}[m] + \mathbf{n}^{(t)}[m],  
\label{eq:system model}
\end{align}
where $\mathbf{x}^{(t)}[m] \in \mathbb{C}^{K\times 1}$ is the $m$-th column vector of $\mathbf{X}^{(t)}$, $\mathbf{n}^{(t)} \sim \mathcal{CN}(\mathbf{0}, \mathbf{I}_N)$ is additive white complex Gaussian noise, $\mathbf{I}_N$ is the $N \times N$ identity matrix, and $\rho$ is the SNR. For $N \ge K$, the BS applies zero-forcing (ZF) decoding to $\mathbf{y}^{(t)}[m]$, and the detected signal $\hat{\mathbf{x}}^{(t)}[m] \in \mathbb{C}^{K\times 1}$ is expressed as
\begin{align}
    \frac{(\mathbf{H}^{(t)\mathrm{H}} \mathbf{H}^{(t)})^{-1} \mathbf{H}^{(t)\mathrm{H}}}{\sqrt{\rho}}\, \mathbf{y}^{(t)}[m] = \mathbf{x}^{(t)}[m] + \tilde{\mathbf{n}}^{(t)}[m],
    \label{eq:decoding}
\end{align}
where $\tilde{\mathbf{n}}^{(t)}[m] \sim \mathcal{CN}\left(\mathbf{0}, (\rho\, \mathbf{H}^{(t)\mathrm{H}} \mathbf{H}^{(t)})^{-1}\right)$. After detection for all $m \in [1,L]$, $\hat{\mathbf{x}}_k^{(t)} \in \mathbb{C}^{1\times L}$ is mapped back to $\hat{\mathbf{g}}_{k_1}^{(t)}$ or $\hat{\mathbf{z}}_{k_2}^{(t)}$ for each $k_1 \in \mathcal{K}_1$ and $k_2 \in \mathcal{K}_2$, corresponding to the recovered gradient and logit, respectively.

\begin{figure}[t]
\centering
\includegraphics[width=\columnwidth]{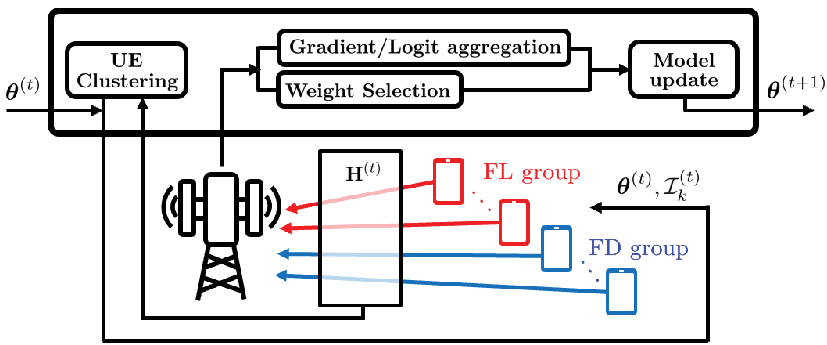}
\caption{HFL framework at round $t$.}
\label{fig:system_model}
\vspace{-1.5em}
\end{figure}

\section{Hybrid Federated Learning}
\label{sec:hfl}
We propose an HFL framework illustrated in Fig.~\ref{fig:system_model} that allows each UE to transmit its gradient or logit based on a channel condition. We first describe the single-round operation and derive convergence, showing that the resulting bound generalizes those of FL and FD. We then introduce dynamic UE clustering and weight selection methods to provide additional DoF. 

\subsection{Hybrid Federated Learning Framework}
At the beginning of communication round $t$, the global model parameter $\boldsymbol{\theta}^{(t)}$ is broadcast to all UEs so that $\boldsymbol{\theta}_k^{(t)} = \boldsymbol{\theta}^{(t)}$. Recall that each UE $k$ belongs to the FL group if $I_k^{(t)} = 0$, and to the FD group if $I_k^{(t)} = 1$. On the UE side, each FL group UE $k_1$ trains locally and obtains the gradient $\mathbf{g}_{k_1}^{(t)} = \nabla F_{k_1} (D_{k_1}^{(t)};\boldsymbol{\theta}_{k_1}^{(t)})$, where $F_{k_1}$ is the cross-entropy (CE) loss and $D_{k_1}^{(t)}$ is UE $k_1$'s local dataset sampled at round $t$ for stochastic gradient descent (SGD). 
Similarly, FD group UE $k_2$ first performs local SGD using its private dataset $D_{k_2}^{(t)}$ to update $\boldsymbol{\theta}_{k_2}^{(t)}$, and then generates logits as $\mathbf{z}_{k_2}^{(t)} = f_{k_2} (D_{\text{pub}}^{(t)};\boldsymbol{\theta}_{k_2}^{(t)})$, where $f_{k_2}$ is the neural network (NN) of UE $k_2$ parameterized by $\boldsymbol{\theta}_{k_2}^{(t)}$ and $D_{\text{pub}}^{(t)}$ is the public dataset shared across all UEs and the BS, sampled at round $t$ for SGD. FL and FD group UEs convert $\mathbf{g}_{k_1}^{(t)}$ and $\mathbf{z}_{k_2}^{(t)}$ into the transmitted signal $\mathbf{X}^{(t)}$, send them via the uplink, and reconstruct $\hat{\mathbf{g}}_{k_1}^{(t)}$ and $\hat{\mathbf{z}}_{k_2}^{(t)}$ as described in Sec.~\ref{Sec:System model}. Recall that $\hat{\mathbf{g}}_{k_1}^{(t)} = \mathbf{g}_{k_1}^{(t)} + \mathbf{e}_{k_1}^{(t)}$ and $\hat{\mathbf{z}}_{k_2}^{(t)} = \mathbf{z}_{k_2}^{(t)} + \mathbf{e}_{k_2}^{(t)}$, where $\mathbf{e}_{k_1}^{(t)}$ and $\mathbf{e}_{k_2}^{(t)}$ are effective noise terms decided by SNR, channel, and preprocessing, applied to each UE's gradient and logit.

Then, at the BS, reconstructed gradients and logits are aggregated as
\begin{align}
\hat{\mathbf{g}}^{(t)} &=  \sum_{k_1 = 1}^{K_1} \frac{|D_{k_1}^{(t)}|}{\sum_{\ell = 1}^{K_1} |D_\ell^{(t)}|} \hat{\mathbf{g}}_{k_1}^{(t)}, \label{eq: weight aggregation}\\
\hat{\mathbf{z}}^{(t)} &= \sum_{k_2 = K_1 + 1}^{K} \frac{|D_{k_2}^{(t)}|}{\sum_{\ell = K_1 + 1}^{K} |D_\ell^{(t)}|} \hat{\mathbf{z}}_{k_2}^{(t)},    
\label{eq: logit aggregation}
\end{align}
which can be equivalently written as $\hat{\mathbf{g}}^{(t)} = \mathbf{g}^{(t)} + \mathbf{e}_{g}^{(t)}$ and $\hat{\mathbf{z}}^{(t)} = \mathbf{z}^{(t)} + \mathbf{e}_{z}^{(t)}$.
The global model is then updated following the convex combination of FL and FD update directions as
\begin{align}
    \boldsymbol{\theta}^{(t+1)} = \boldsymbol{\theta}^{(t)} - \alpha \eta_1 \hat{\mathbf{g}}^{(t)} - (1-\alpha) \eta_2 \nabla Q(D_{\text{pub}}^{(t)};\boldsymbol{\theta}^{(t)}, \hat{\mathbf{z}}^{(t)}),
\label{eq:update rule}
\end{align}
where $Q(D_{\text{pub}}^{(t)};\boldsymbol{\theta}^{(t)}, \hat{\mathbf{z}}^{(t)}) = \text{KL}( \hat{\mathbf{z}}^{(t)}/\tau \,\Vert\, f(D_{\text{pub}}^{(t)};\boldsymbol{\theta}^{(t)})/\tau )$
and $\text{KL}(\cdot\Vert\cdot)$ denotes the Kullback–Leibler divergence (KLD). Here, $\eta_1$ and $\eta_2$ are learning rates, $f$ is the NN of BS, and $\tau$ is a temperature parameter controlling distribution smoothness. FL and FD emerge as special cases of HFL under noise-free uplink, i.e., when $\mathbf{e}_{g}^{(t)} = \mathbf{e}_{z}^{(t)} = \mathbf{0}$. Specifically, when $\alpha = 1$ and $K = K_1$, HFL reduces to FL \cite{McMahan:2017}, whereas when $\alpha = 0$ and $K_1 = 0$, it reduces to FD \cite{Liu:2022}.

\subsection{Convergence Analysis}
We first state the assumptions used to establish the convergence of the proposed HFL framework. 
\begin{asum}
    \textbf{(Convexity)}: Let $F(D;\boldsymbol{\theta}) = \sum_k \lambda_k F_k(D_k;\boldsymbol{\theta})$, where $D$ is the total dataset, $0<\lambda_k < 1$ with $\sum_k \lambda_k = 1$. The function $F(D;\boldsymbol{\theta})$ is $\mu_1$-strongly convex for all $\boldsymbol{\theta}_1, \boldsymbol{\theta}_2$, i.e.,
    \begin{align}
        F(D;\boldsymbol{\theta}_1) &\ge F(D;\boldsymbol{\theta}_2) 
        + \langle\boldsymbol{\theta}_1 - \boldsymbol{\theta}_2, \nabla F(D;\boldsymbol{\theta}_2)\rangle \nonumber \\
        &\quad + \frac{\mu_1}{2} \lVert \boldsymbol{\theta}_1 - \boldsymbol{\theta}_2 \rVert^2, 
    \end{align}
    where $\langle\cdot, \cdot\rangle$ denotes inner product.
    Similarly, $Q(D;\boldsymbol{\theta},\mathbf{z})$, which is defined after \eqref{eq:update rule}, is $\mu_2$-strongly convex as
    \begin{align}
        Q(D;\boldsymbol{\theta}_1,\mathbf{z}) &\ge Q(D;\boldsymbol{\theta}_2,\mathbf{z}) 
        + \langle\boldsymbol{\theta}_1 - \boldsymbol{\theta}_2, \nabla Q(D;\boldsymbol{\theta}_2,\mathbf{z})\rangle \nonumber \\
        &\quad + \frac{\mu_2}{2} \lVert \boldsymbol{\theta}_1 - \boldsymbol{\theta}_2 \rVert^2.
    \end{align}
    \label{assumption1}
\end{asum}
\vspace{-2.0em}
\begin{asum}
    \textbf{(Bounded variance)}: Let $D^{(t)}$ denote the dataset sampled at round $t$. The variance of the stochastic gradient with mini-batches is upper bounded by some $\psi_1^2$ and $\psi_2^2$ for $F$ and $Q$, respectively.
    \begin{align}
        \mathbb{E}\!\left[\lVert \nabla F(D^{(t)};\boldsymbol{\theta}) - \nabla F(D;\boldsymbol{\theta}) \rVert^2\right] \le \psi_1^2,\\
        \mathbb{E}\!\left[\lVert \nabla Q(D^{(t)};\boldsymbol{\theta},\mathbf{z}) - \nabla Q(D;\boldsymbol{\theta},\mathbf{z}) \rVert^2\right] \le \psi_2^2,
    \end{align}
    where $D^{(t)} \subset D$.
    \label{assumption2}
\end{asum}

\begin{asum}
    \textbf{(Bounded gradient)}: The gradient norms are upper bounded by $G_1^2$ and $G_2^2$ for $F$ and $Q$, respectively.
    \begin{align}
        \mathbb{E}\!\left[\lVert \nabla F(D;\boldsymbol{\theta}) \rVert^2\right] \le G_1^2, \\
        \mathbb{E}\!\left[\lVert \nabla Q(D;\boldsymbol{\theta},\mathbf{z}) \rVert^2\right] \le G_2^2.
    \end{align}
    \label{assumption3}
\end{asum}
\vspace{-2.0em}
\noindent These assumptions are commonly used in FL convergence analyses \cite{Xing:2021}. We additionally assume the following to handle logit noise.
\begin{asum}
    \textbf{(Lipschitz continuity in logit)}: The gradient of $Q$ is Lipschitz-continuous in $\mathbf{z}$.
    \begin{align}
        \lVert \nabla Q(D;\boldsymbol{\theta}, \mathbf{z}_1) -  \nabla Q(D;\boldsymbol{\theta}, \mathbf{z}_2) \rVert \le L\lVert \mathbf{z}_1 - \mathbf{z}_2 \rVert.
    \end{align}
    \label{assumption 4}
\end{asum}
\vspace{-2.0em }
\noindent Based on these assumptions, we establish the convergence of the proposed HFL framework.

\setcounter{thm}{0}
\begin{prop}
As $t \rightarrow \infty$, assuming small $\mathbf{e}_{z}^{(t)}$ and $\mathbb{E}[\mathbf{e}_{g}^{(t)}] = \mathbb{E}[\mathbf{e}_{z}^{(t)}] = \mathbf{0}$, the sequence $\{\boldsymbol{\theta}^{(t)}\}$ trained via HFL satisfies $\mathbb{E}\!\left[\lVert \boldsymbol{\theta}^{(t)} - \boldsymbol{\theta}^{*} \rVert^2\right] \le A/\bar{\mu}$, where $\bar{\mu}$ and $A$ are constants, with $0 < \bar{\mu} < 1$.
\end{prop}

\begin{proof}
Define $\sigma_g$, $\sigma_z$, $\boldsymbol{\theta}^*$, and $\boldsymbol{\delta}_q^{(t)}$ as follows. Let $\sigma_g = \sup_{t} \mathbb{E}[\lVert \mathbf{e}_g^{(t)} \rVert^2]$ and $\sigma_z = \sup_{t} \mathbb{E}[\lVert \mathbf{e}_z^{(t)} \rVert^2]$. The optimal parameter $\boldsymbol{\theta}^*$ is assumed to satisfy $F(D^{(t)};\boldsymbol{\theta}^*) \le F(D^{(t)};\boldsymbol{\theta}^{(t)})$ and $Q(D^{(t)};\boldsymbol{\theta}^*,\mathbf{z}^{(t)}) \le Q(D^{(t)};\boldsymbol{\theta}^{(t)},\mathbf{z}^{(t)})$ for all $\boldsymbol{\theta}^{(t)}$, $D^{(t)}$, and $\mathbf{z}^{(t)}$. Finally, $\boldsymbol{\delta}_{Q}^{(t)}$ is defined as 
\begin{align}
    \boldsymbol{\delta}_{Q}^{(t)} = \nabla Q(D_{\text{pub}}^{(t)};\boldsymbol{\theta}^{(t)}, \hat{\mathbf{z}}^{(t)}) - \nabla Q(D_{\text{pub}}^{(t)};\boldsymbol{\theta}^{(t)}, \mathbf{z}^{(t)}).    
\end{align}
With small $\mathbf{e}_z^{(t)}$, a first-order Taylor expansion gives $\boldsymbol{\delta}_{Q}^{(t)} \approx \mathbf{e}_{z}^{(t)}\frac{\partial}{\partial \mathbf{z}}\nabla Q(D_{\text{pub}}^{(t)};\boldsymbol{\theta}^{(t)}, \mathbf{z}) \rvert_{\mathbf{z} = \mathbf{z}^{(t)}}$, implying $\mathbb{E}[\boldsymbol{\delta}_{Q}^{(t)}] = \mathbf{0}$. By Assumption~\ref{assumption 4}, $\lVert \boldsymbol{\delta}_Q^{(t)} \rVert$ is upper bounded by $L\lVert \mathbf{e}_z^{(t)} \rVert$.

The update rule \eqref{eq:update rule} is then rewritten as 
\begin{align}
    \boldsymbol{\theta}^{(t+1)} = &\boldsymbol{\theta}^{(t)} - \alpha \eta_1 (\nabla F(D^{(t)};\boldsymbol{\theta}^{(t)}) + \mathbf{e}_g^{(t)}) \nonumber\\&- (1-\alpha) \eta_2 (\nabla Q(D_{\text{pub}}^{(t)};\boldsymbol{\theta}^{(t)}, \mathbf{z}^{(t)}) + \boldsymbol{\delta}_Q^{(t)}).
\end{align}
To derive the upper bound of $\mathbb{E}[\lVert \boldsymbol{\theta}^{(t+1)} - \boldsymbol{\theta}^{*}\rVert^2]$, the derivation proceeds as follows. First, by Assumption~\ref{assumption1} and the definition of $\boldsymbol{\theta}^*$, 
\begin{align}
    &-2\alpha\eta_1 \mathbb{E}[\langle\boldsymbol{\theta}^{(t)} - \boldsymbol{\theta}^*, \nabla F(D^{(t)};\boldsymbol{\theta}^{(t)})\rangle] \nonumber\\&- 2(1-\alpha)\eta_2 \mathbb{E}[\langle\boldsymbol{\theta}^{(t)} - \boldsymbol{\theta}^*, \nabla Q(D^{(t)};\boldsymbol{\theta}^{(t)},\mathbf{z}^{(t)})\rangle] \nonumber\\ &\le  -\alpha \eta_1 \mu_1 \mathbb{E}[\lVert \boldsymbol{\theta}^{(t)} - \boldsymbol{\theta}^* \rVert^2] -(1-\alpha) \eta_2 \mu_2 \mathbb{E}[\lVert \boldsymbol{\theta}^{(t)} - \boldsymbol{\theta}^* \rVert^2].
\end{align}
Second, by Assumptions~\ref{assumption2} and \ref{assumption3}, 
\begin{align}
    &\mathbb{E}[\lVert \nabla F(D^{(t)};\boldsymbol{\theta}^{(t)}) - \nabla F(D;\boldsymbol{\theta}^{(t)}) +\nabla F(D;\boldsymbol{\theta}^{(t)})\rVert^2 \nonumber\\
    &\le \psi_1^2 + G_1^2.
\end{align}
The assumptions also yield $\mathbb{E}[\lVert \nabla Q(D^{(t)};\boldsymbol{\theta}^{(t)},\mathbf{z}^{(t)}) \rVert^2] \le \psi_2^2 + G_2^2$. Finally, by Cauchy-Schwarz inequality, the average inner product of the two gradients is bounded by $\sqrt{(\psi_1^2 + G_1^2)(\psi_2^2 + G_2^2)}$. Combining these, we obtain 
\begin{align}
    &\mathbb{E}[\lVert \boldsymbol{\theta}^{(t+1)} - \boldsymbol{\theta}^* \rVert^2] \nonumber\\&\le(1-\alpha\eta_1\mu_1 - (1-\alpha)\eta_2\mu_2)\mathbb{E}[\lVert \boldsymbol{\theta}^{(t)} - \boldsymbol{\theta}^* \rVert^2] \nonumber\\
    &+\alpha^2\eta_1^2 (\psi_1^2 + G_1^2 + \sigma_g) + (1-\alpha)^2 \eta_2^2 (\psi_2^2 + G_2^2 + L^2\sigma_z) \nonumber \\
    &+ 2\alpha(1-\alpha)\eta_1\eta_2\sqrt{(\psi_1^2 + G_1^2)(\psi_2^2 + G_2^2)} \nonumber\\
    & = (1-\bar{\mu})\mathbb{E}[\lVert \boldsymbol{\theta}^{(t)} - \boldsymbol{\theta}^* \rVert^2] + A.
    \label{ieq:upper bound}
\end{align} 
When $0 < \bar{\mu} < 1$ by adjusting $\eta_1, \eta_2$, and $\alpha$, an upper bound $A/\bar{\mu}$ is derived as $t \rightarrow \infty$.
\end{proof}
\noindent Again, the upper bound recovers FL and FD as special cases, which are derived in \cite{Mu:2024}. When $\sigma_g = \sigma_z = 0$ and $\alpha = 1$, we obtain the FL bound. Otherwise, we obtain the FD bound when $\sigma_g = \sigma_z = 0$ and $\alpha = 0$. Although the convergence analysis for the proposed HFL framework is derived under a small logit error assumption and therefore does not strictly apply to low-SNR regimes, the empirical results in Sec.~\ref{Sec:Simulation results} show that, by appropriately exploiting the DoFs of HFL as introduced in the next subsection, the algorithm still converges at low SNR, exhibiting behavior consistent with the small logit error regime.

\subsection{Degrees of Freedom: Adaptive UE Clustering and Weight Selection}
In each communication round of HFL, the BS determines two configurations: the assignment of UEs to FL group or FD group, i.e., gradient or logit transmissions, and the weight used to fuse the resulting FL and FD global updates. These decisions exploit instantaneous channel state information (CSI) and loss feedback on public data, assuming that the decisions implemented via lightweight computations can be completed within the channel coherence time.

\subsubsection{Adaptive UE clustering}
For each UE $k$, the BS computes the post-equalization noise enhancement factor $q_k = 1/(\rho[ \mathbf{H}^{(t)\mathrm{H}}\mathbf{H}^{(t)}]_{kk})$. A smaller $q_k$ indicates more reliable gradient transmission, whereas a larger $q_k$ suggests that transmitting logits is preferable due to their noise robustness. To partition UEs into FL and FD groups, we apply the Jenks natural breaks optimization \cite{Jenks:1967} with $S=2$ classes to the set $\{q_k\}_{k=1}^{K}$. Jenks optimization minimizes intra-class variance and identifies a threshold $q^\star$ separating low and high noise-enhancement factors. This approach is equivalent to one-dimensional $k$-means clustering \cite{Heim:2022}. The BS then assigns UE $k$ to the FL group $\mathcal{I}_{k}^{(t)} = 0$ if $q_k \le q^\star$, and to the FD group $\mathcal{I}_{k}^{(t)} = 1$ if $q_k > q^\star$.

\subsubsection{Adaptive weight selection}
After computing the FL and FD updates, their respective global gradient descent directions are defined as $\mathbf{d}_{\mathrm{fl}} = -\eta_1 \hat{\mathbf{g}}^{(t)}$ and $\mathbf{d}_{\mathrm{fd}} = - \eta_2 \nabla Q(D_{\text{pub}}^{(t)};\boldsymbol{\theta}^{(t)}, \hat{\mathbf{z}}^{(t)})$. To combine these two directions, we perform an optimization over convex weights $(\sigma(s^{(w)}), 1-\sigma(s^{(w)}))$, where $\sigma(\cdot)$ is the sigmoid function and $s^{(w)} \in \mathbb{R}$ is an unconstrained scalar updated at epoch $w$. The objective is to minimize the cross-entropy loss on public data
\begin{align}
\mathcal{L}(s^{(w)}) &=  
F(D_{\text{pub}}^{(t)}; \boldsymbol{\theta}^{(t)} + \sigma(s^{(w)}) \mathbf{d}_{\text{fl}} + (1-\sigma(s^{(w)})) \mathbf{d}_{\text{fd}}).
\label{eq: damped newton method loss function}
\end{align}
We solve this minimization using a damped Newton method due to its rapid convergence in a few iterations \cite{Boyd:2004} as
\begin{align}
s^{(w+1)} = s^{(w)} - \eta_3 \,
\frac{\partial \mathcal{L}(s^{(w)}) / \partial s^{(w)}}
{\partial^2 \mathcal{L}(s^{(w)}) / \partial (s^{(w)})^2},
\label{eq:damped newton method}
\end{align}
where the first- and second-order derivatives are approximated via central finite differences, and $\eta_3$ is the damping factor. The final value $s^\star$ determines the weight $\alpha = \sigma(s^*)$ and is recomputed each round. 

The introduced lightweight methods enable HFL to adaptively reflect instantaneous channel conditions every round. A summary of the overall HFL procedure is provided in Algorithm~1.

\begin{algorithm}[t]
\caption{Hybrid Federated Learning}
\begin{algorithmic}[1]
\For{$t \in \{ 1, ..., T\}$} 
\State \textbf{[BS]}
\State Decide $\mathcal{I}_{k}^{(t)}$ for $\forall k$ via Jenks optimization method with acquired CSI $\mathbf{H}_{t}$.
\State Broadcast $\boldsymbol{\theta}^{(t)}$, $\mathcal{I}_{k}^{(t)}$ to all UEs.
\State \textbf{[UE]}
\State Train UE to obtain gradient $\mathbf{g}_k^{(t)}$.
\State For FL group UEs, convert $\mathbf{g}_{k_1}^{(t)}$ as $\mathbf{x}_{k_1}^{(t)}$. For FD group UEs, using trained $\boldsymbol{\theta}_{k_2}^{(t)}$, generate logit $\mathbf{z}_{k_2}^{(t)}$. Then, convert $\mathbf{z}_{k_2}^{(t)}$ as $\mathbf{x}_{k_2}^{(t)}$.
\For{$m \in \{ 1, ..., L\}$} 
\State \textbf{[UE]}
\State Transmit $\mathbf{x}^{(t)}[m]$ through uplink channel.
\State \textbf{[BS]}
\State Conduct ZF to decode each signal as \eqref{eq:decoding}.
\EndFor
\State \textbf{[BS]}
\State Reconstruct signal to gradients and logits as $\hat{\mathbf{g}}_{k_1}^{(t)}$ and logits $\hat{\mathbf{z}}_{k_2}^{(t)}$.
\State Aggregate parameters and logits to $\hat{\mathbf{g}}^{(t)}$ and $\hat{\mathbf{z}}^{(t)}$ using \eqref{eq: weight aggregation} and \eqref{eq: logit aggregation}.
\State Optimize FL/FD weights using damped Newton method as \eqref{eq:damped newton method} by setting the loss function as \eqref{eq: damped newton method loss function}.
\EndFor
\label{Alg:Channel-aware HFL}
\end{algorithmic}
\end{algorithm}

\begin{figure}[t]
\centering
    \begin{subfigure}{\columnwidth}
    \centering
    \includegraphics[width=\columnwidth]{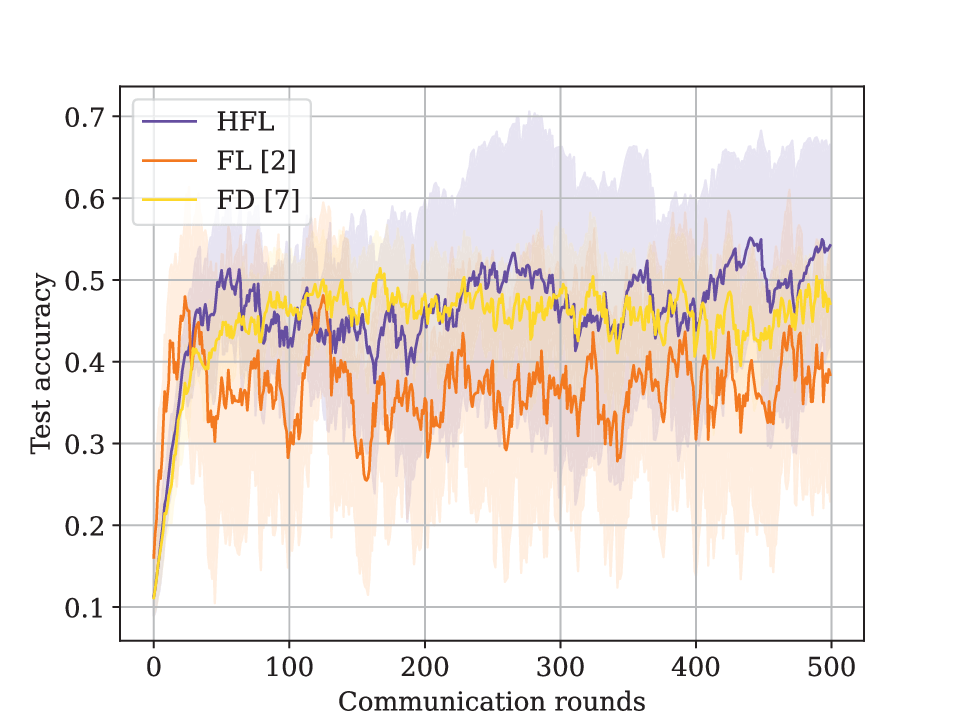}
    \caption{\centering Comparison of HFL, FL, and FD at $\rho= -20\text{dB}$.}
    \label{sim_fig:HFL: -20db}
    \end{subfigure}
    \begin{subfigure}{\columnwidth}
    \centering
     \includegraphics[width=\columnwidth]{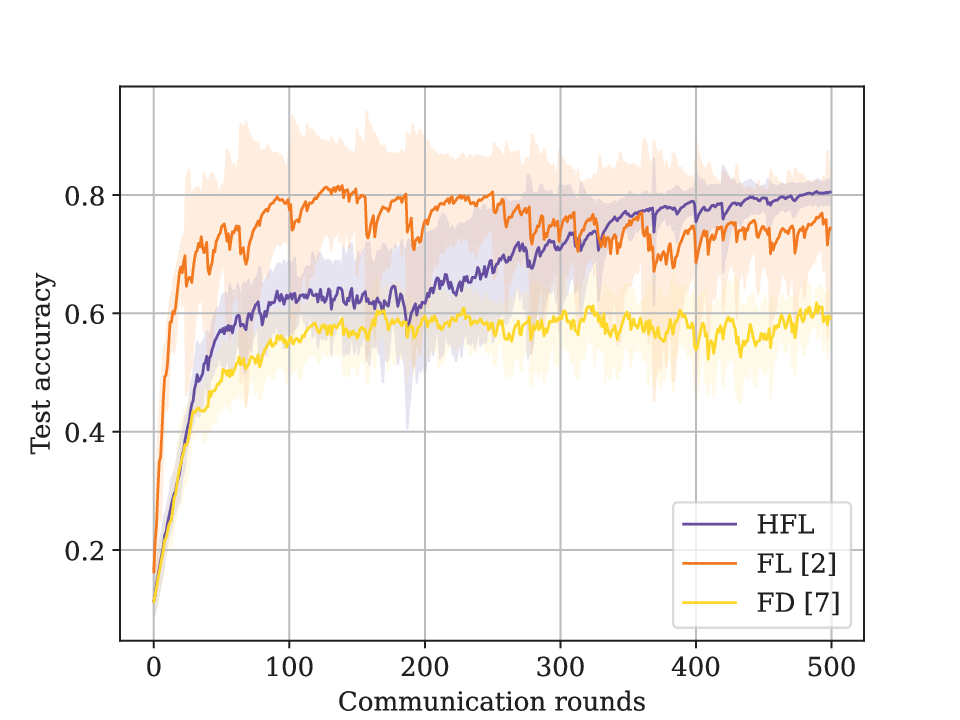}
    \caption{\centering Comparison of HFL, FL, and FD under at $\rho = -15\text{dB}$.}
    \label{sim_fig:HFL: -15db}
    \end{subfigure}
\caption{\centering Comparison of HFL, FL, and FD at low SNR.}
\label{sim_fig:HFL vs FL vs FD}
\end{figure}

\begin{figure}[t]
\centering
    \begin{subfigure}{\columnwidth}
    \centering
    \includegraphics[width=\columnwidth]{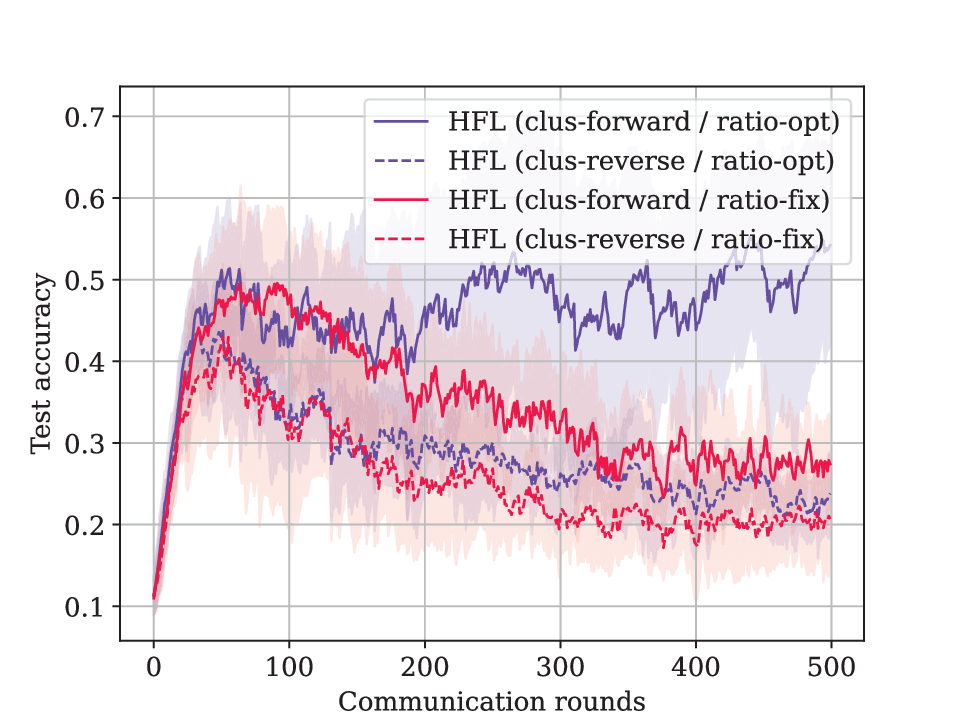}
    \caption{\centering Effect of DoF at $\rho = -20\text{dB}$.}
    \label{sim_fig:HFL_DoF_-20db}
    \end{subfigure}
    \vspace{0.1em}
    \begin{subfigure}{\columnwidth}
    \centering
     \includegraphics[width=\columnwidth]{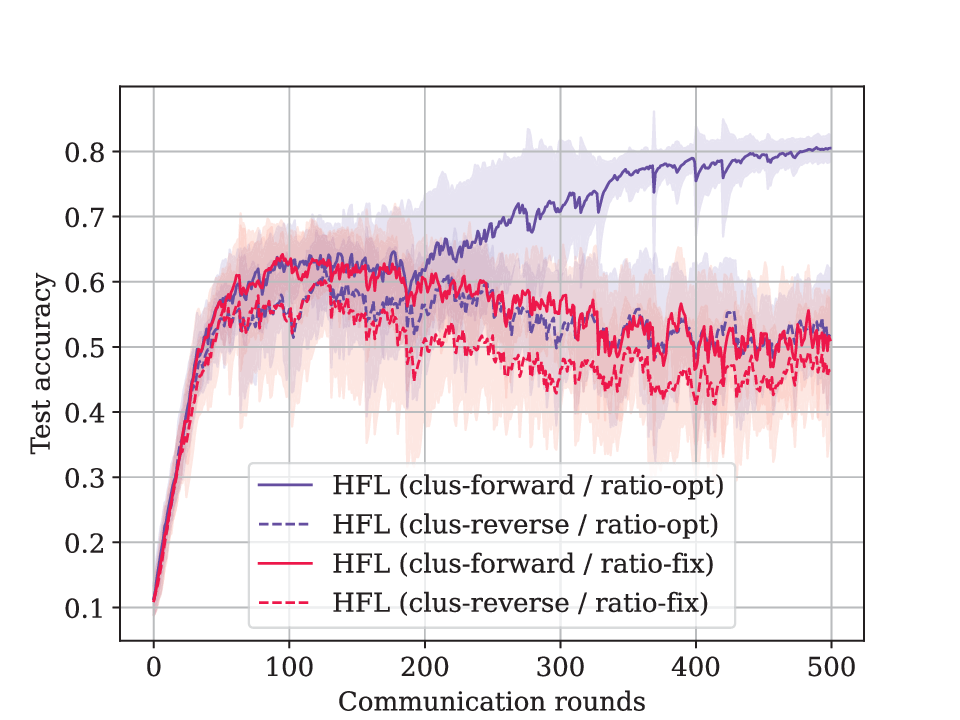}
    \caption{\centering Effect of DoF at $\rho = -15\text{dB}$.}
    \label{sim_fig:HFL_DoF_-15db}
    \end{subfigure}
\caption{\centering Effect of DoF at low SNR: Clustering and ratio selection.}
\vspace{-1.0em} 
\label{sim_fig:HFL_DoF}
\end{figure}

\section{Simulation results}
\label{Sec:Simulation results}
In this section, we show the performance of the proposed method by comparing it with FL and FD at low SNR. We then validate the offered DoF. In our system, we consider $N=30$ antennas at the BS and $K = 30$ UEs fully participating in distributed learning without scheduling to solve the MNIST classification problem with $C=10$ classes. There is no zero-padding at the transmission stage by setting $L = P/2 = CP_{\text{pub}}/2 = 39755$. To emphasize the noise-robust HFL training, we set low SNR $\rho$ as $-20$dB and $-15$dB, then evaluate the test accuracy every communication round. For model training, learning rates and damping factor are set as $\eta_1 = \eta_2 = 0.01$, and $\eta_3 = 0.1$. Local training and KD epochs are $1$, damped Newton method epochs are $30$, and the temperature is $\tau = 2$.

\subsection{Compare FL, FD, and HFL}
In Fig.~\ref{sim_fig:HFL vs FL vs FD}, we compare the test accuracy of FL, FD, and HFL at low SNRs, which are $\rho = -20\text{dB}$ and $\rho = -15\text{dB}$. FL uses FedAvg \cite{McMahan:2017}, and FD follows \cite{Liu:2022}. At the harsher $\rho=-20 \text{dB}$, FL achieves lower accuracy than FD, consistent with its sensitivity to channel noise that corrupts high-dimensional gradients prior to aggregation. FD is less affected because only low-dimensional logits are transmitted. HFL attains the highest accuracy and converges faster than FD. When the SNR improves to $\rho=-15 \text{dB}$, FL recovers and surpasses FD but remains below HFL after convergence, while FD continues to converge more slowly and saturates at a lower accuracy. HFL shows the highest accuracy but converges more slowly than FL. Across both SNRs, HFL achieves the highest test accuracy by combining gradient- and logit-based updates, inheriting the FD’s noise-robustness while leveraging the FL's learning performance.

\subsection{Effect of DoF}
In Fig.~\ref{sim_fig:HFL_DoF}, we evaluate HFL under four configurations to validate the proposed DoF selection scheme: UE clustering (`clus-forward', `clus-reverse') and weight selection (`weight-opt', `weight-fix'). For clustering, UEs are partitioned adaptively using the noise-enhancement factor as the criterion. `clus-forward' follows our proposed rule, whereas `clus-reverse' applies the opposite mapping, where the BS sets $\mathcal{I}_{k}^{(t)}=0$ if $q_k>q^\star$ and $\mathcal{I}_{k}^{(t)}=1$ if $q_k\le q^\star$. For weight selection, `weight-opt' adjusts the weight each round, while `weight-fix' holds it at $\sigma(s^*)=0.5$. As expected, using both `clus-forward' and `weight-opt' achieves the highest test accuracy, demonstrating its robustness to noise. In addition, the superior accuracy of `clus-forward' over `clus-reverse' further supports our heuristic assignment of high-noise UEs to FL and low-noise UEs to FD.

\section{Conclusion}
\label{Sec:Conclusion}
In this paper, we proposed a novel HFL framework that adapts its training strategy to the wireless channel conditions, mitigating the complementary weaknesses of FL and FD. We first described the overall framework and derived its convergence. Then, we exploited the HFL's DoF through adaptive UE clustering and weight selection, which determine per UE whether gradients or logits are transmitted, as well as the weighting of FL and FD contributions, handled by Jenks optimization and a damped Newton method, respectively. Despite its simplicity, the HFL differs from prior works by enabling per-UE, channel-aware selection of the transmitted information. Numerical results showed consistent gains of HFL at low SNR when fully exploiting its DoF. Future work includes designing beamformers beyond ZF for over-the-air HFL and introducing additional DoFs informed by both channel conditions and local data distributions.

\bibliographystyle{IEEEtran}
\bibliography{references}

\end{document}